\pdfoutput=1
\documentclass[letterpaper]{article} 
\usepackage{aaai2027}  
\usepackage[hyphens]{url}  
\usepackage{graphicx} 
\usepackage{natbib}  
\usepackage{caption} 
\usepackage{pifont}
\usepackage{algorithm}
\usepackage{amsmath}
\usepackage{algpseudocode}
\usepackage{booktabs}
\usepackage[table]{xcolor}
\usepackage{array}
\usepackage{subcaption}
\usepackage{multirow}
\usepackage{makecell}
\usepackage{amsthm}
\usepackage{amssymb}
\usepackage{mathrsfs}

\newtheorem{theorem}{Theorem}

\usepackage{newfloat}
\usepackage{listings}
\DeclareCaptionStyle{ruled}{labelfont=normalfont,labelsep=colon,strut=off} 
\floatstyle{ruled}
\newfloat{listing}{tb}{lst}{}
\floatname{listing}{Listing}

\title{When Teachers Mislead: Spurious-Signal-Aware On-Policy Distillation}
\author{
    Yinuo Jiang\textsuperscript{\rm 1,\rm 2}\thanks{This work was done during an internship at ByteDance.},
    Yongjie Ye\textsuperscript{\rm 2},
    Zhou Tao\textsuperscript{\rm 2},
    Xiang Zhuang\textsuperscript{\rm 3},\\
    Qiang Zhang\textsuperscript{\rm 1},
    Huajun Chen\textsuperscript{\rm 1}\corresponding,
    Tiankai Li\textsuperscript{\rm 2}\corresponding
}
\affiliations{
    \textsuperscript{\rm 1}Zhejiang University,
    \textsuperscript{\rm 2}ByteDance,\\
    \textsuperscript{\rm 3}Shanghai Artificial Intelligence Laboratory\\[3pt]
    yinuoj@zju.edu.cn,
    huajunsir@zju.edu.cn,
    litiankai.jason@bytedance.com\\[5pt]
    \textbf{SA-OPD Code: }\url{https://github.com/jjjyinuo/SA-OPD}
}

\nocopyright
\begin{document}

\maketitle

\begin{abstract}
 On-Policy distillation (OPD) transfers teacher capabilities by supervising
student-sampled trajectories with dense token-level teacher signals. Recent
selective OPD methods improve this process by prioritizing signals that are
confident, informative, or learnable. However, the assumptions overlook a fundamental failure mode of language models: their token-level judgments can be driven by input-agnostic language priors, formatting conventions, or stereotyped reasoning templates rather than task-specific evidence. We refer to such optimization-relevant but weakly input-grounded supervision as \textbf{spurious signals} in OPD, which may produce large gradients while contributing little task-improving direction. To mitigate this issue, we propose \textbf{SA-OPD}, a \textbf{S}purious-Signal-\textbf{A}ware \textbf{O}n-\textbf{P}olicy \textbf{D}istillation framework that identifies and filters misleading token-level supervision based on input-groundedness and optimization impact. SA-OPD introduces a lightweight input-groundedness proxy estimating whether a token-level distillation signal truly depends on the input. It then filters only tokens that simultaneously exhibit low input-groundedness and extreme distillation divergence, thereby removing high-impact spurious updates and achieving fine-grained OPD optimization. Extensive experiments on
both large language model (LLM) and vision-language model (VLM) settings demonstrate that
SA-OPD consistently outperforms Vanilla OPD and competitive selective methods. These
results establish input-groundedness as a key dimension for OPD supervision
selection and offer a simple, effective strategy for mitigating spurious updates.
\end{abstract}
 \section{Introduction}

    \begin{figure*}[t]
        \centering
        \includegraphics[width=0.98\textwidth]{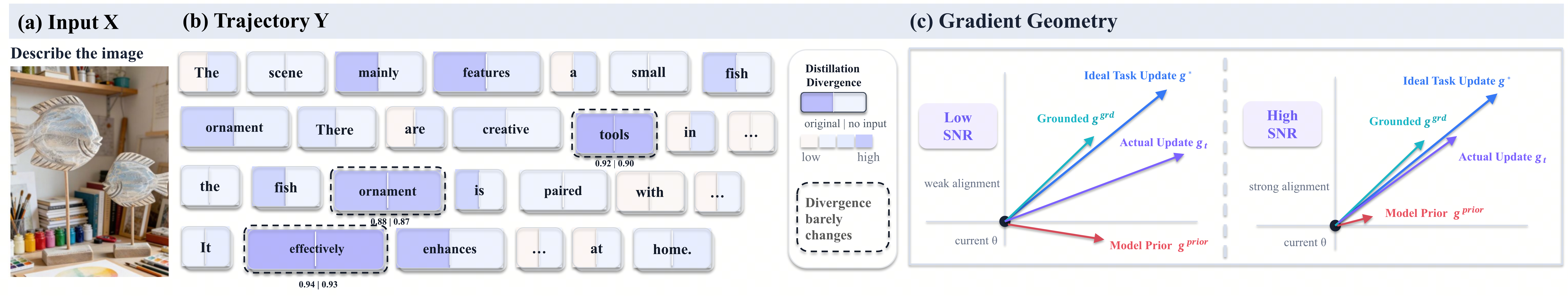}
        \caption{Illustration of spurious signals in OPD. Panels (a) and (b) present a concrete example: (a) shows the input image and question, while (b) visualizes the token-level teacher--student divergence along the corresponding student rollout. Some tokens exhibit nearly unchanged teacher--student divergence after prompt removal, suggesting that their dense supervision is driven by generic entity collocations or prior templates rather than input-specific visual evidence. Panel (c) provides the gradient interpretation: such weakly grounded signals induce low-SNR OPD updates, with large prior-driven gradient components but weak alignment with the ideal task-improving direction.}
        \label{fig:problem}
    \end{figure*}

    Knowledge distillation transfers capabilities from a strong teacher to a smaller student. Conventional distillation often relies on teacher-generated or offline trajectories, which can create a distribution mismatch between training states and the student's own evolving policy. On-policy distillation (OPD)~\citep{lu2025onpolicydistillation,gu2026minillmonpolicydistillationlarge} mitigates this mismatch by rolling out from the current student policy and applying dense token-level teacher supervision on student-induced states, improving distillation efficiency beyond sparse sequence-level rewards. Despite this progress, recent studies have shown that not every teacher signal in OPD is equally useful. OPD can be improved by selecting or reweighting supervision according to token uncertainty, teacher confidence, teacher--student divergence, or local learnability~\citep{xu2026tiptokenimportanceonpolicy,wang2026disagreementlearnabletokenteachability,jin2026entropyawareonpolicydistillationlanguage,li2026filterreweightrethinkingoptimization}. These methods reveal an important principle: OPD should not blindly imitate all teacher signals. Instead, it should prioritize signals that are informative, learnable, and likely to induce useful student updates.

    However, existing selection criteria still share a fundamental assumption: they treat the teacher judgment itself as reliable once it is confident, salient, or learnable. This assumption overlooks a basic property of teacher models: \emph{a teacher is itself a learned model whose token-level judgments are shaped not only by the current task input, but also by language priors, formatting conventions, and stereotyped reasoning patterns acquired during large-scale pretraining and post-training}. As a result, the teacher may assign a biased reward or penalty to a token not because the token is appropriate for the specific input, but because it matches or violates a generic template. We refer to such input-ungrounded yet optimization-relevant judgments as \emph{spurious signals} in OPD. These signals are especially harmful because they can produce large gradients while contributing little task-improving direction. For example, a teacher may over-reward generic reasoning phrases, answer templates, or repeated textual patterns even when they are not supported by the current image or question. Conversely, it may over-penalize a valid student reasoning path simply because it deviates from the teacher's preferred style. In both cases, such signals induce spurious OPD gradients: \textbf{the student is not merely learning from the teacher, but also inheriting the teacher's input-agnostic biases}. Yet this failure mode is invisible to common OPD diagnostics. Entropy, teacher--student divergence, and thinking-pattern consistency measure how strong, different, or learnable a supervision signal is, but not whether it is grounded in the input.

    Guided by these insights, we propose \textbf{SA-OPD}, a \textbf{S}purious-Signal-\textbf{A}ware \textbf{O}n-\textbf{P}olicy \textbf{D}istillation framework. The key idea is to evaluate whether a token-level teacher--student divergence is actually grounded in the input. For each student-generated trajectory, SA-OPD computes the divergence under two conditions: the original input context and a residual no-prompt context that keeps the same generated prefix while removing the task input. If the divergence barely changes after removing the input, the signal is likely explained by language priors or template preferences rather than input-specific evidence. We use this difference as a lightweight proxy for the \emph{input-groundedness} of the distillation signal. Finally, SA-OPD filters tokens that satisfy two conditions simultaneously: \textbf{low input-groundedness and high absolute teacher--student divergence}, targeting the most harmful part of dense OPD supervision: high-impact updates that are weakly supported by the input. We evaluate SA-OPD in both large language model and vision-language model settings. Across mathematical reasoning, visual understanding, and visual reasoning benchmarks, SA-OPD consistently improves over Vanilla OPD and competitive selective OPD baselines. To sum up, our contributions are threefold:
    \begin{itemize}
        \item We identify and formalize a previously underexplored failure mode in OPD: high-impact distillation signals that are weakly grounded in the input and therefore induce spurious student updates.
        \item We introduce SA-OPD, a practical filtering framework that removes only tokens with both low input-groundedness and high optimization impact, requiring no external verification labels.
        \item Extensive experiments on LLM and VLM distillation demonstrate that SA-OPD consistently outperforms Vanilla OPD and strong selective OPD baselines across mathematical reasoning, visual understanding, and visual reasoning benchmarks.
    \end{itemize}

    \section{Related Work}

    \subsection{On-Policy Distillation}

    OPD has recently emerged as an effective paradigm for post-training. Prior studies show that reverse-KL-style objectives and supervision on student-generated mistakes can improve open-ended generation and reasoning tasks~\citep{gu2026minillmonpolicydistillationlarge,agarwal2024onpolicydistillationlanguagemodels}. Recent work further studies how to make OPD scalable, stable, and generalizable through reward extrapolation, entropy-aware objectives, reasoning-prefix acceleration, competence-aware curricula, divergence constraints, and rollout mixture distillation~\citep{yang2026learningteachergeneralizedonpolicy,zhang2026fasteffectiveonpolicydistillation,luo2026demystifyingopdlengthinflation,hou2026uniopdunifyingonpolicydistillation,jin2026entropyawareonpolicydistillationlanguage,zheng2026scopesignalcalibratedonpolicydistillation,li2026filterreweightrethinkingoptimization}. Meanwhile, OPD has also been extended to self-distillation~\citep{zhao2026selfdistilledreasoneronpolicyselfdistillation,he2026selfdistillationzeroselfrevisionturns,wang2026skillsdskillconditionedselfdistillationmultiturn}, hybrid RL-distillation frameworks~\citep{yan2025learningreasonoffpolicyguidance,zhang2026reinforcementawareknowledgedistillationllm,yang2026selfdistilledrlvr}, and multimodal distillation~\citep{yuan2026visionopdlearningfinedetails,li2026videoopdefficientposttrainingmultimodal}.

    \subsection{Token-level Selection in Post-training}

    Token-level selection has become an increasingly important perspective in post-training, as recent studies show that not all tokens contribute equally to policy improvement. In reinforcement learning, prior studies identify high-entropy forking tokens, advantage--log-probability covariance, decision-critical branch points, and overconfident errors as key drivers of policy improvement or collapse~\citep{wang20258020rulehighentropyminority,cui2025entropymechanismreinforcementlearning,cheng2026reasoning}. These findings suggest that post-training should emphasize tokens corresponding to uncertain decisions, high-impact branches, or confident mistakes, rather than treating all tokens uniformly. A parallel line of work has recently emerged in OPD, where dense teacher supervision makes token-wise allocation especially important. TIP selects tokens that are either high-entropy or low-entropy but highly divergent from the teacher~\citep{xu2026tiptokenimportanceonpolicy}. Token Teachability further filters teacher--student disagreements by their local learnability~\citep{wang2026disagreementlearnabletokenteachability}; Entropy-Aware OPD adapts the KL direction at high-teacher-entropy positions to preserve diversity~\citep{jin2026entropyawareonpolicydistillationlanguage}; FiRe-OPD combines trajectory filtering with token-level soft reweighting to emphasize informative tokens without discarding supervision entirely~\citep{li2026filterreweightrethinkingoptimization}. Collectively, these methods show that OPD benefits from selective dense supervision.

    However, existing criteria mainly ask whether a token is uncertain, divergent, or learnable, but not whether the teacher's judgment is grounded in the current input. Thus, a token may appear informative under prior criteria while its teacher signal is still driven by language priors, formatting conventions, or stereotyped reasoning templates. Our work complements prior token-level selection by filtering high-impact teacher signals that are weakly input-grounded.

    \begin{figure*}[t]
        \centering
        \includegraphics[width=0.98\textwidth]{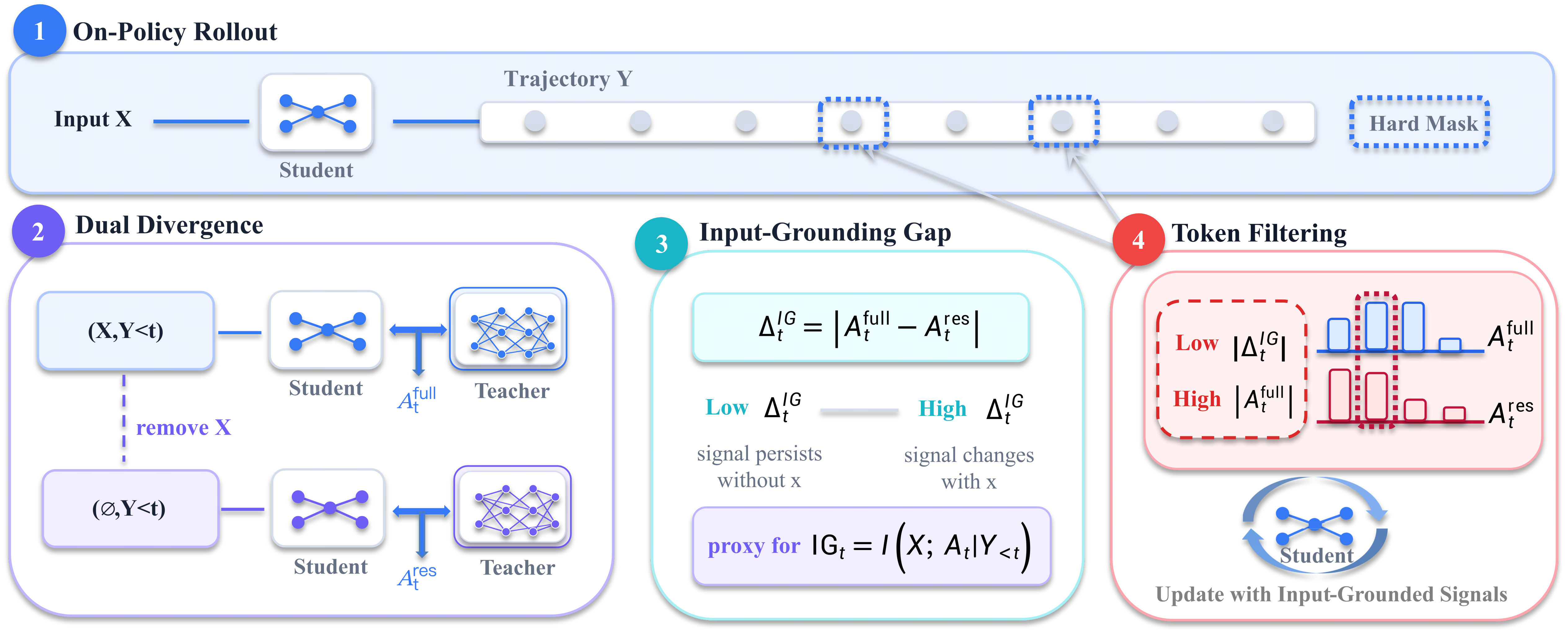}
        \caption{Overview of the proposed SA-OPD framework. Given an on-policy rollout, SA-OPD computes teacher--student divergence under the original prompt and a residual no-prompt context. If the divergence barely changes after prompt removal, the teacher signal is likely dominated by input-agnostic priors. SA-OPD filters tokens that are both weakly input-grounded and high-impact, enabling OPD to update from input-grounded distillation signals rather than spurious signals.}
        \label{fig:sa_opd_overview}
    \end{figure*}

    \section{Method}
    \label{sec:method}

    \subsection{Spurious Signals in OPD}
    \label{sec:spurious-signals}

    We formalize the spurious token-level signals that motivate our filtering criterion. Following the Vanilla OPD setting~\citep{lu2025onpolicydistillation}, let $\pi_{\rm T}$ denote a frozen teacher and $\pi_\theta$ a trainable student. Given an input $x$, the student generates an on-policy response $y=(y_1,\ldots,y_L)\sim\pi_\theta(\cdot\mid x)$. At position $t$, OPD compares the student and teacher distributions at the student-visited prefix $(x,y_{<t})$. The standard reverse-KL objective is
    \begin{equation}
    \mathcal{L}_{\rm OPD}
    =
    \frac{1}{L}
    \sum_{t=1}^{L}
    D_{\rm KL}
    \left(
    \pi_\theta(\cdot\mid x,y_{<t})
    \,\middle\|\,
    \pi_{\rm T}(\cdot\mid x,y_{<t})
    \right).
    \label{eq:opd-reverse-kl}
    \end{equation}

    For the student-sampled token $y_t\sim\pi_\theta(\cdot\mid x,y_{<t})$, define the sampled \emph{teacher-student divergence} as
    \begin{equation}
    A_t
    =
    \log\pi_\theta(y_t\mid x,y_{<t})
    -
    \log\pi_{\rm T}(y_t\mid x,y_{<t}).
    \label{eq:full-reverse-kl-sample}
    \end{equation}
    Treating this coefficient as a stop-gradient weight, the corresponding token-level update is proportional to
    \begin{equation}
    g_t^{\rm OPD}
    =
    -A_t
    \nabla_\theta
    \log\pi_\theta(y_t\mid x,y_{<t}).
    \label{eq:opd-token-gradient}
    \end{equation}

    \paragraph{Input-Grounded and Prior-Driven Signals.}
    To distinguish input-dependent supervision from generic language preferences, we conceptually decompose the token-level divergence as
    \begin{equation}
    A_t
    =
    A_t^{\rm grd}
    +
    A_t^{\rm prior},
    \label{eq:sampled-signal-decomposition}
    \end{equation}
    where $A_t^{\rm grd}$ denotes the component that varies with the current input, whereas $A_t^{\rm prior}$ denotes the component predictable from the response prefix alone.

    Given $s_t=\nabla_\theta
  \log \pi_\theta(y_t\mid x,y_{<t})$, the corresponding OPD update decomposes linearly as

  \begin{equation}
  \begin{aligned}
  g_t^{\rm OPD}
  &=
  -A_t s_t \\
  &=
  -A_t^{\rm grd}s_t
  -A_t^{\rm prior}s_t \\
  &=
  g_t^{\rm grd}
  +
  g_t^{\rm prior}.
  \end{aligned}
  \label{eq:opd-gradient-decomposition}
  \end{equation}
    \paragraph{Alignment with the Task Objective.}
    Let $g_t^\star$ denote an ideal token-level update that improves the downstream task objective. The usefulness of an OPD update is determined not only by its magnitude but also by its alignment with this direction. A reliable distillation update should satisfy
    \begin{equation}
    \mathbb{E}
    \left[
    \left\langle
    g_t^{\rm OPD},g_t^\star
    \right\rangle
    \right]
    >0.
    \label{eq:positive-gradient-alignment}
    \end{equation}
    Using Eq.~\eqref{eq:opd-gradient-decomposition}, its expected alignment can be decomposed as
    \begin{align}
    \mathbb{E}
    \left[
    \left\langle
    g_t^{\rm OPD},g_t^\star
    \right\rangle
    \right]
    &=
    \mathbb{E}
    \left[
    \left\langle
    g_t^{\rm grd},g_t^\star
    \right\rangle
    \right]
    +
    \mathbb{E}
    \left[
    \left\langle
    g_t^{\rm prior},g_t^\star
    \right\rangle
    \right].
    \label{eq:gradient-alignment-decomposition}
    \end{align}
    The grounded component can provide task-relevant supervision because it changes with the input. We therefore expect its coherent component to be positively aligned with the desired update:
    \begin{equation}
    \mathbb{E}
    \left[
    \left\langle
    g_t^{\rm grd},g_t^\star
    \right\rangle
    \right]
    >0.
    \label{eq:grounded-gradient-alignment}
    \end{equation}
    In contrast, when the divergence is dominated by the prior component, the induced gradient may have large magnitude but little alignment with the task-improving direction (Appendix~A.1):
    \begin{equation}
    \mathbb{E}
    \left[
    \left\|
    g_t^{\rm prior}
    \right\|^2
    \right]
    \gg 0,
    \qquad
    \mathbb{E}
    \left[
    \left\langle
    g_t^{\rm prior},g_t^\star
    \right\rangle
    \right]
    \approx 0.
    \label{eq:prior-gradient-no-alignment}
    \end{equation}
    Such updates behave as nuisance gradients: they contribute little coherent task-improving direction, but can substantially increase stochastic gradient energy.

    \paragraph{An Effective Gradient SNR.}
    The preceding decomposition motivates an effective token-level gradient signal-to-noise ratio. We regard the coherent grounded update as signal and the energy of the approximately zero-mean prior update as noise:
    \begin{equation}
    {\rm SNR}_t
    =
    \frac{
    \left\|
    \mathbb{E}
    \left[
    g_t^{\rm grd}
    \right]
    \right\|^2
    }{
    \mathbb{E}
    \left[
    \left\|
    g_t^{\rm prior}
    \right\|^2
    \right]
    }.
    \label{eq:token-snr}
    \end{equation}
    The numerator captures the coherent grounded update after averaging over tokens and contexts, while the denominator captures the second-moment energy of the prior-induced update. Thus, this quantity behaves like a stochastic-gradient SNR when prior-driven variation dominates gradient noise. Input-insensitive tokens have weak grounded signal because their teacher--student disagreement changes little with the input. However, they can still induce large prior-driven updates when the teacher and student strongly disagree on generic tokens or templates (Figure~\ref{fig:problem}). Such tokens therefore yield low-SNR gradients, provide little stable task-improving direction, and induce parameter drift (Appendix~A.2).
    \paragraph{Spurious Signals.}
    \label{para:Spurious-Signals}
    The SNR view suggests that harmful OPD tokens are those whose gradient power is dominated by input-agnostic components rather than input-grounded supervision. We use \emph{Input-Groundedness} as a proxy for this property. Specifically, we quantify the input dependence of the teacher--student divergence signal by the conditional mutual information
    \begin{equation}
    {\rm IG}_t
    =
    I\!\left(X; A_t \mid Y_{<t}\right).
    \label{eq:input-groundedness}
    \end{equation}
    A large ${\rm IG}_t$ indicates that the token-level divergence changes with the input and is therefore more likely to correspond to a high-SNR, input-grounded update. Conversely, a small ${\rm IG}_t$ suggests that the signal is largely predictable from the prefix alone and is therefore at risk of being dominated by language priors or template preferences.

    Finally, spurious signals in OPD should not be defined by low input-groundedness alone. Following the signal-to-curvature view of token importance~\citep{xu2026tiptokenimportanceonpolicy}, a harmful OPD signal should also have high decision impact: its first-order update must be large enough to affect the local optimization dynamics relative to the curvature cost. We therefore define spurious OPD signals as the intersection of two conditions: \textbf{low input-groundedness and high update impact}. This definition targets the most harmful part of dense OPD supervision---large updates that are weakly supported by the input.

    \subsection{Spurious-Signal-Aware OPD (SA-OPD)}

    Motivated by the above diagnosis, SA-OPD turns spurious-signal detection into a practical token-filtering procedure, as illustrated in Figure~\ref{fig:sa_opd_overview}.

    \paragraph{Estimating Input-Groundedness.}
    \label{sec:mi-proxy}

    The mutual information in Eq.~\eqref{eq:input-groundedness} has no closed form for high-dimensional token sequences. We therefore introduce a lightweight empirical proxy that measures whether the distillation signal remains distinguishable when the prompt is removed. Our construction is based on the intuition that an input-grounded distillation signal should be sensitive to the task input: if the input changes or is removed, the teacher-induced signal should change accordingly. For each token, we conduct a dual teacher--student divergence computation under two conditions:
    \begin{equation}
    A_t^{\rm full}
    =
    A_t(x,y_{\leq t}),
    \label{eq:full-signal}
    \end{equation}
    and
    \begin{equation}
    A_t^{\rm res}
    =
    A_t(\varnothing,y_{\leq t}),
    \label{eq:res-signal}
    \end{equation}
    where $A_t^{\rm full}$ is computed with the original prompt and $A_t^{\rm res}$ is computed after removing the prompt while keeping the same student-generated response prefix. The no-prompt signal approximates the component of the teacher judgment that can be explained by generic language priors, formatting conventions, and stereotyped reasoning patterns. We then define the \emph{Input-Grounding Gap}:
    \begin{equation}
    \Delta_t^{\rm IG}
    =
    \left|
    A_t^{\rm full}
    -
    A_t^{\rm res}
    \right|.
    \label{eq:ig-gap}
    \end{equation}
    A large $\Delta_t^{\rm IG}$ indicates that the teacher-induced signal differs substantially between the real-input and no-prompt conditions, suggesting that the judgment is grounded in the specific input. A small $\Delta_t^{\rm IG}$ indicates that the teacher produces nearly the same signal without seeing the prompt, suggesting that the update is dominated by template priors.

    \paragraph{Token-level Filtering.}

    As discussed in Section~\ref{sec:spurious-signals}, we define spurious OPD signals as the co-occurrence of two properties: low input-groundedness and high optimization impact. Specifically, we define a token-level filtering indicator as
    \begin{equation}
    \text{TokenFiltered}_t
    =
    \mathbf{1}
    \left[
    \Delta_t^{\rm IG} < \tau_{\rm IG}
    \right]
    \cdot
    \mathbf{1}
    \left[
    |A_t^{\rm full}| > \tau_A
    \right],
    \label{eq:token-filter-rule}
    \end{equation}
    where $\tau_{\rm IG}$ controls sensitivity to input dependence and $\tau_A$ removes low-impact template-like tokens. The first term in Eq.~\eqref{eq:token-filter-rule} identifies tokens whose teacher judgment is nearly invariant to the input. The second term ensures that we remove only teacher signals that exert a large reward or penalty on the student. In practice, we implement the token filter with a top-$p$ selection rule. For each batch, we select tokens that simultaneously have low input-groundedness and large absolute distillation signal, where $\operatorname{Bottom}_{p_1}(\Delta_t^{\rm IG})$ denotes the tokens with the lowest $p_1$ fraction of input-groundedness scores, and $\operatorname{Top}_{p_2}(|A_t^{\rm full}|)$ denotes the tokens with the largest $p_2$ fraction of absolute distillation signals:
    \begin{equation}
    \mathcal{F}(p_1,p_2)
    =
    \operatorname{Bottom}_{p_1}
    \left(
    \Delta_t^{\rm IG}
    \right)
    \cap
    \operatorname{Top}_{p_2}
    \left(
    |A_t^{\rm full}|
    \right).
    \label{eq:top-p-filter}
    \end{equation}
    We further define the filtered loss-mass ratio (FLMR) of a candidate set $\mathcal{F}$ as
    \begin{equation}
    \mathrm{FLMR}(\mathcal{F})
    =
    \frac{
    \sum_{t\in\mathcal{F}}
    |A_t^{\rm full}|
    }{
    \sum_{t\in\mathcal{V}}
    |A_t^{\rm full}|+\epsilon
    },
    \label{eq:flmr}
    \end{equation}
    where $\mathcal{V}$ denotes all valid response tokens in the batch. FLMR measures how much of the current OPD loss mass would be removed by the filter. Unlike the filtered token ratio, which only counts the number of removed tokens, FLMR reflects the optimization weight of the removed supervision.

    A fixed pair of selection ratios $(p_1,p_2)$ can be brittle across tasks and models. In unstable tasks, overly aggressive filtering may discard valid teacher signals along with spurious ones, weakening useful OPD supervision. We therefore dynamically adapt $p_1$ and $p_2$ during training to keep the removed loss mass bounded, with implementation details provided in Appendix~C:
    \begin{equation}
    \mathrm{FLMR}\!\left(\mathcal{F}(p_1,p_2)\right)
    \leq
    \beta.
    \label{eq:dynamic_beta}
    \end{equation}
    The final filtered tokens are those in $\mathcal{F}(p_1^\star,p_2^\star)$ after adaptation. This mechanism preserves the core criterion---low input dependence and high optimization impact---while preventing the filter from removing an excessive fraction of useful distillation signal.

    Given the final filtered token set $\mathcal{F}$, SA-OPD optimizes the reverse-KL objective only on the retained tokens:
    \begin{equation}
    \mathcal{L}_{\rm SA}
    =
    \frac{1}{|\mathcal{V}\setminus\mathcal{F}|}
    \sum_{\substack{
    t\in\mathcal{V}\\
    t\notin\mathcal{F}
    }}
    D_{\rm KL}
    \left(
    \pi_\theta(\cdot\mid x,y_{<t})
    \,\middle\|\,
    \pi_{\rm T}(\cdot\mid x,y_{<t})
    \right).
    \label{eq:sa-opd-loss}
    \end{equation}

    \section{Experiments}

    \subsection{Settings}

    \subsubsection{Models and Training Data.}
    We conduct all experiments on non-thinking variants from the Qwen3~\citep{yang2025qwen3technicalreport} and Qwen3.5~\citep{qwen3.5} families as both teacher and student policies. Specifically, the main experiments and all analyses are conducted on the Qwen3-4B-Instruct $\rightarrow$ Qwen3-1.7B pair and the Qwen3.5-35B-A3B $\rightarrow$ Qwen3.5-2B pair. To verify the generalization ability of our method, we also evaluate the DeepSeek-R1-0528-Qwen3-8B~\citep{deepseekai2025deepseekr1incentivizingreasoningcapability} $\rightarrow$ Qwen3-1.7B pair and the Qwen3.5-9B $\rightarrow$ Qwen3.5-2B pair. For LLM-based OPD, we use DeepMath~\citep{he2025deepmath103klargescalechallengingdecontaminated} as the source dataset. We keep examples with difficulty level at least 6 and randomly sample 30\% of the remaining data, yielding about 7K math-reasoning training examples. For VLM-based OPD, we construct visual-understanding data by sampling 10\% of the Captioning \& IF, Grounding, and Counting \& Search subsets from VERO-600K~\citep{sarch2026veroopenrlrecipe}, and construct visual-reasoning data by sampling 10\% of MMRL30k~\citep{zhu2026shuffler1efficientrlframework}. More training details are provided in Appendix~B.

    \subsubsection{Evaluation.}
    For VLM-based OPD, we include six benchmarks covering two aspects: (1) visual understanding: EvoChart~\citep{huang2025evochartbenchmarkselftrainingapproach}, MMIFEval~\citep{ding2025mmifenginemultimodalinstructionfollowing}, and CountQA~\citep{tamarapalli2025countqamllmscountwild}; and (2) visual reasoning: MathVision~\citep{wang2024measuringmultimodalmathematicalreasoning}, Geo3K~\citep{lu2021intergpsinterpretablegeometryproblem}, and MathVista~\citep{lu2024mathvistaevaluatingmathematicalreasoning}. For LLM-based OPD, we evaluate mathematical reasoning performance on five benchmarks spanning a range of difficulty levels: AIME 2024, AIME 2025, Math500~\citep{hendrycks2021measuringmathematicalproblemsolving}, AMC 2023, and MinervaMATH~\citep{lewkowycz2022solvingquantitativereasoningproblems}.  All benchmarks are evaluated using the official metrics and evaluations to ensure consistent comparison; evaluation details are provided in Appendix~B.

    \subsubsection{Baselines.}
    We compare SA-OPD against \ding{182} standard OPD methods: Vanilla OPD~\citep{lu2025onpolicydistillation} and ExOPD~\citep{yang2026learningteachergeneralizedonpolicy}; and \ding{183} selective OPD methods: TIP~\citep{xu2026tiptokenimportanceonpolicy} and FiRe-OPD~\citep{li2026filterreweightrethinkingoptimization}. All methods are trained under the same data, model, and compute budget for fair comparison.

    \subsection{Main Results}

    \begin{table*}[t]
    \centering
    \setlength{\tabcolsep}{4.2pt}
    \renewcommand{\arraystretch}{1.08}
    \definecolor{bestblue}{RGB}{232,242,255}
    \begin{tabular}{lcccccccc}
    \toprule
    \multirow{2}{*}{\textbf{Method}}
    & \multicolumn{4}{c}{\textbf{Visual Understanding}}
    & \multicolumn{4}{c}{\textbf{Visual Reasoning}} \\
    \cmidrule(lr){2-5}\cmidrule(lr){6-9}
    & \textbf{EvoChart}
    & \textbf{MMIFEval}
    & \textbf{CountQA}
    & \textbf{Avg.}
    & \textbf{MathVision}
    & \textbf{Geo3K}
    & \textbf{MathVista}
    & \textbf{Avg.} \\
    \midrule
    Student
    & 69.1 & 50.6 & 19.8 & 46.5
    & 38.4 & 63.4 & 63.9 & 55.2 \\
    Teacher
    & 82.0 & 65.0 & 63.0 & 70.0
    & 68.2 & 84.9 & 73.6 & 75.6 \\
    \midrule
    OPD
    & 71.6 & 53.5 & 26.4 & 50.5
    & \underline{45.0} & 67.2 & 69.0 & 60.4 \\
    ExOPD
    & \underline{72.8} & 52.9 & 23.4 & 49.7
    & 43.0 & 68.8 & \underline{71.7} & 61.2 \\
    FiRe-OPD
    & 72.4 & 53.2 & \underline{30.8} & 52.1
    & 42.2 & 66.0 & 69.8 & 59.3 \\
    TIP
    & 72.2 & \underline{54.4} & 30.2 & \underline{52.3}
    & 44.4 & \underline{70.0} & 70.5 & \underline{61.6} \\
    \midrule
    \rowcolor{bestblue}
    \textbf{SA-OPD}
    & \textbf{73.2}
    & \textbf{55.2}
    & \textbf{33.6}
    & \textbf{54.0}
    & \textbf{46.2}
    & \textbf{72.2}
    & \textbf{72.2}
    & \textbf{63.5} \\
    $\Delta$ over OPD
    & +1.6 & +1.7 & +7.2 & +3.5
    & +1.2 & +5.0 & +3.2 & +3.1 \\
    \bottomrule
    \end{tabular}
    \caption{Results on visual understanding and visual reasoning benchmarks under Qwen3.5-35B-A3B $\rightarrow$ Qwen3.5-2B. Best and second-best results among OPD methods are shown in \textbf{bold} and \underline{underlined}, respectively.}
    \label{tab:visual_results}
    \end{table*}

    \begin{table}[t]
    \centering
    \small
    \setlength{\tabcolsep}{1.0pt}
    \renewcommand{\arraystretch}{1.05}
    \definecolor{bestblue}{RGB}{220,235,255}
    \begin{tabular}{lcccccc}
    \toprule
    \textbf{Method}
    & \multicolumn{6}{c}{\textbf{Math Reasoning}} \\
    \cmidrule(lr){2-7}
    & \textbf{Math500}
    & \textbf{AMC23}
    & \textbf{AIME25}
    & \textbf{AIME24}
    & \textbf{Minerva}
    & \textbf{Avg.} \\
    \midrule
    Student
    & 36.4 & 16.6 & 0.8 & 1.7 & 6.7 & 12.4 \\
    Teacher
    & 80.8 & 93.8 & 46.2 & 63.3 & 27.7 & 62.4 \\
    \midrule
    OPD
    & 66.4 & 43.8 & 6.2 & 9.2 & 16.7 & 28.5 \\
    ExOPD
    & 66.2 & \textbf{44.7} & 4.6 & \underline{11.2} & \underline{17.6} & 28.9 \\
    FiRe-OPD
    & 65.8 & 43.2 & \textbf{8.3} & \underline{11.2} & 17.4 & 29.2 \\
    TIP
    & \underline{66.6} & 44.1 & \underline{7.5} & \textbf{11.7} & 16.8 & \underline{29.3} \\
    \midrule
    \rowcolor{bestblue}
    \textbf{SA-OPD}
    & \textbf{69.6}
    & \textbf{44.7}
    & \textbf{8.3}
    & \textbf{11.7}
    & \textbf{17.8}
    & \textbf{30.4} \\
    $\Delta$ over OPD
    & +3.2 & +0.9 & +2.1 & +2.5 & +1.1 & +1.9 \\
    \bottomrule
    \end{tabular}
    \caption{Results on math reasoning benchmarks under Qwen3-4B-Instruct $\rightarrow$ Qwen3-1.7B. Best and second-best results among OPD methods are shown in \textbf{bold} and \underline{underlined}, respectively.}
    \label{tab:math_results}
    \end{table}

    \paragraph{Performance.}
    Tables~\ref{tab:visual_results} and~\ref{tab:math_results} show that SA-OPD consistently improves OPD across both VLM and LLM settings. On visual understanding and visual reasoning benchmarks, SA-OPD achieves the best result on all six tasks. Compared with Vanilla OPD, SA-OPD improves the average score from 50.5 to 54.0 on visual understanding and from 60.4 to 63.5 on visual reasoning. Moreover, SA-OPD also consistently outperforms the strongest prior OPD variant on every visual benchmark, with notable gains on CountQA (+2.8), Geo3K (+2.2), and MathVision (+1.2). The same trend holds for mathematical reasoning. SA-OPD obtains the best or tied-best result on all five benchmarks and achieves the highest average score, improving Vanilla OPD from 28.5 to 30.4 and TIP from 29.3 to 30.4. Notably, SA-OPD improves Math500 by 3.0 points over the strongest prior OPD baseline. Overall, SA-OPD provides a more reliable distillation objective by preserving useful teacher knowledge while suppressing spurious high-impact updates. We provide computation overhead analysis in Appendix~D.

    \paragraph{Scalability.}
    We further evaluate SA-OPD under different teacher--student size settings to examine whether its effectiveness depends on a specific model scale or teacher strength. Table~\ref{tab:opd_vertical_results} shows that SA-OPD consistently yields positive gains across model sizes in both VLM and LLM settings, suggesting strong robustness to teacher--student scale variations.

    \begin{table}[t]
    \centering
    \small
    \setlength{\tabcolsep}{1.5pt}
    \renewcommand{\arraystretch}{1.12}
    \definecolor{bestblue}{RGB}{220,235,255}
    \definecolor{groupgray}{RGB}{245,245,245}
    \begin{tabular}{p{0.22\columnwidth}lcccc}
    \toprule
    \textbf{Model Pair}
    & \textbf{Benchmark}
    & \textbf{Student}
    & \textbf{Teacher}
    & \textbf{OPD}
    & \textbf{SA-OPD} \\
    \midrule
    \rowcolor{groupgray}
    \multicolumn{6}{l}{\textit{VLM Distillation}} \\
    \multirow{6}{0.22\columnwidth}{
    \centering
    \makecell[l]{
    Qwen3.5-9B\\
    $\qquad\downarrow$\\
    Qwen3.5-2B
    }}
    & \makecell[r]{EvoChart} & 28.1 & 80.2 & 70.8 & \cellcolor{bestblue}\textbf{74.8} \\
    & \makecell[r]{MMIFEval} & 50.6 & 67.4 & 54.4 & \cellcolor{bestblue}\textbf{55.7} \\
    & \makecell[r]{CountQA} & 19.8 & 58.2 & 32.2 & \cellcolor{bestblue}\textbf{35.6} \\
    & \makecell[r]{MathVision} & 38.4 & 78.9 & 43.4 & \cellcolor{bestblue}\textbf{46.6} \\
    & \makecell[r]{Geo3K} & 63.4 & 80.2 & 67.7 & \cellcolor{bestblue}\textbf{69.1} \\
    & \makecell[r]{MathVista} & 63.9 & 68.8 & 69.0 & \cellcolor{bestblue}\textbf{69.1} \\
    \midrule
    \rowcolor{groupgray}
    \multicolumn{6}{l}{\textit{LLM Distillation}} \\
    \multirow{5}{0.22\columnwidth}{
    \centering
    \makecell[l]{
    R1-Qwen3-8B\\
    $\qquad\downarrow$\\
    Qwen3-1.7B
    }}
    & \makecell[r]{Math500} & 36.4 & 65.0 & 58.4 & \cellcolor{bestblue}\textbf{60.6} \\
    & \makecell[r]{AMC23} & 16.6 & 76.6 & 35.3 & \cellcolor{bestblue}\textbf{37.8} \\
    & \makecell[r]{AIME25} & 0.8 & 28.8 & 5.4 & \cellcolor{bestblue}\textbf{6.3} \\
    & \makecell[r]{AIME24} & 1.7 & 38.3 & 5.8 & \cellcolor{bestblue}\textbf{7.5} \\
    & \makecell[r]{Minerva} & 6.7 & 29.0 & 13.9 & \cellcolor{bestblue}\textbf{14.7} \\
    \bottomrule
    \end{tabular}
    \caption{Scalability comparison between SA-OPD and Vanilla OPD. R1-Qwen3-8B denotes DeepSeek-R1-0528-Qwen3-8B. Best results are shown in \textbf{bold}.}
    \label{tab:opd_vertical_results}
    \end{table}

    \subsection{Ablations}

    \begin{table}[t]
    \centering
    \small
    \setlength{\tabcolsep}{5pt}
    \renewcommand{\arraystretch}{1.08}
    \begin{tabular}{lcc}
    \toprule
    \textbf{Filter Variant} & \textbf{Geo3K} & \textbf{MathVista} \\
    \midrule
    \emph{Base} & 67.2 & 69.0 \\
    \midrule
    \emph{Random Filter} & 67.6 & 67.8 \\
    \emph{Divergence-only Filter} & 69.8 & 69.6 \\
    \emph{Input-Groundedness-only Filter} & 69.0 & 70.3 \\
    \emph{Proxy Replacement} & 69.6 & 71.5 \\
    \midrule
    \textbf{SA-OPD} & \textbf{72.2} & \textbf{72.1} \\
    \bottomrule
    \end{tabular}
    \caption{Ablation study of SA-OPD.}
    \label{tab:ablation}
    \end{table}

    We ablate the key components of our filtering strategy in Table~\ref{tab:ablation}. For a fair comparison, all filtering-based variants are calibrated to remove approximately the same proportion of tokens, with the filtering ratio controlled within about 1 percentage point (pp). We compare: \ding{182} \textbf{Random Filter}, which randomly removes the same proportion of tokens; \ding{183} \textbf{Divergence-only Filter}, which filters tokens solely by extreme teacher--student divergence, without considering whether the signal is input-grounded; \ding{184} \textbf{Input-Groundedness-only Filter}, which filters only low-input-groundedness tokens without considering their optimization impact; and \ding{185} \textbf{Proxy Replacement}, which replaces the input-dependent component in the input-groundedness proxy with the corresponding teacher log-probability.

    The ablation results show that both axes of our criterion are necessary. Divergence-only filtering improves over the base model, but remains inferior to SA-OPD because large divergence may also correspond to useful, input-grounded correction signals. Input-Groundedness-only filtering also helps, especially on MathVista, but is suboptimal since many low-groundedness tokens have limited optimization impact. Proxy Replacement performs better than the single-criterion variants, yet still lags behind SA-OPD, suggesting that teacher confidence is less aligned with the actual OPD update signal. Overall, SA-OPD performs best on both benchmarks, confirming that harmful tokens are better identified by combining low input-groundedness with high optimization impact.

    \subsection{Additional Analyses}
    \begin{figure}[t]
    \centering
    \includegraphics[width=0.99\columnwidth]{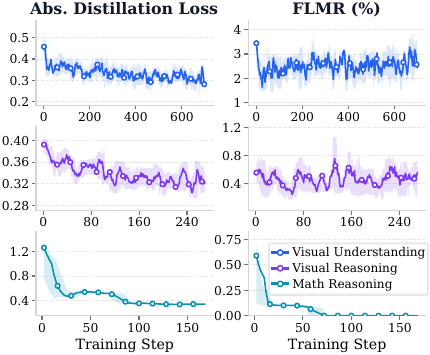}
    \caption{Training dynamics across tasks.}
    \label{fig:td}
    \end{figure}

    \subsubsection{Training Differences across Tasks.}
    Figure~\ref{fig:td} reveals a clear difference between LLM-based and VLM-based distillation. In math reasoning, the teacher--student divergence is large only at the beginning of training: the absolute distillation-loss change drops rapidly and then remains nearly flat. In contrast, both visual understanding and visual reasoning exhibit persistent fluctuations throughout training, indicating that VLM OPD receives a much less stationary teacher signal. This is expected: VLM teachers must jointly resolve visual grounding, perception uncertainty, and language priors, making their token-level judgments more susceptible to visually plausible but input-agnostic shortcuts~\citep{guan2023hallusionbench,wang2024mitigatinghallucinationslargevisionlanguage}.

    To quantify how much of the OPD objective is occupied by such spurious signals during training, we track FLMR, defined in Eq.~\eqref{eq:flmr}. Since FLMR weights filtered tokens by their absolute distillation signal, it reflects the fraction of the gradient budget associated with weakly input-grounded, high-impact supervision, rather than merely counting how many tokens are removed. The results show clear differences in spurious signals across tasks. For LLM mathematical reasoning, FLMR drops rapidly to nearly zero in the early stage of training, suggesting that high-impact spurious signals are mainly concentrated in the initial alignment phase. Once the student distribution stabilizes, such signals largely disappear. In contrast, VLM tasks maintain a non-zero filtered loss mass throughout training, indicating that spurious supervision is not merely transient optimization noise, but a persistent source of dense supervision contamination. This also explains why SA-OPD brings larger gains on VLM tasks: when OPD gradients are continuously and structurally contaminated by spurious signals, input-dependence-based filtering can more effectively improve the signal-to-noise ratio of training.

    \subsubsection{Hyper-parameter Analysis.}

    \begin{figure}[t]
    \centering
    \begin{minipage}[t]{0.5\columnwidth}
    \centering
    \includegraphics[width=\linewidth]{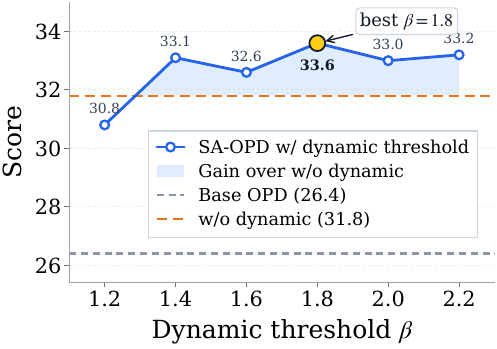}
    \small (a) Sensitivity study on the dynamic threshold $\beta$.
    \end{minipage}
    \hfill
    \begin{minipage}[t]{0.49\columnwidth}
    \centering
    \includegraphics[width=\linewidth]{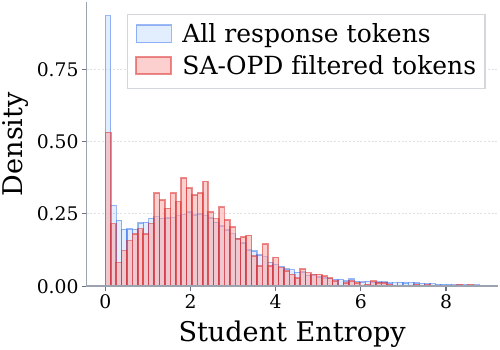}
    \small (b) Entropy distribution in the visual understanding setting.
    \end{minipage}
    \caption{Analysis of SA-OPD's dynamic threshold and filtered-token entropy distribution.}
    \label{fig:beta-and-entropy}
    \end{figure}

    For visual reasoning and mathematical reasoning tasks, the FLMR under fixed selection ratios $p_1$ and $p_2$ is relatively small, around 1pp--2pp at the beginning of training. In contrast, the FLMR in visual understanding tasks is much higher, indicating that the filtered tokens account for a larger portion of the loss mass. If the filtering is too aggressive, it may mistakenly remove tokens that still carry useful perceptual information. To reduce this risk, we adopt the dynamic threshold adjustment in Eq.~\eqref{eq:dynamic_beta}. Results in Figure~\ref{fig:beta-and-entropy}(a) show that when $\beta$ is not too low, dynamic thresholding consistently improves performance. The best result is achieved at $\beta=1.8$, where the counting score increases from 30.8 to 33.6. This suggests that imposing a moderate constraint on FLMR can better balance the suppression of spurious signals and the preservation of effective visual supervision.

    \subsubsection{Filtered Token Analysis.}

    \begin{table}[t]
    \centering
    \small
    \setlength{\tabcolsep}{5pt}
    \renewcommand{\arraystretch}{1.08}
    \begin{tabular}{lcc}
    \toprule
    \textbf{Category / Property} & \textbf{Count} & \textbf{Ratio} \\
    \midrule
    \rowcolor{gray!12}
    \multicolumn{3}{l}{\textit{All Filtered Tokens}} \\
    Content word & 1375 & 69.9\% \\
    Punctuation / format & 407 & 20.7\% \\
    Function word & 121 & 6.2\% \\
    Reasoning-template word & 30 & 1.5\% \\
    Number & 28 & 1.4\% \\
    \midrule
    \rowcolor{gray!12}
    \multicolumn{3}{l}{\textit{Content Tokens}} \\
    Token seen before & 1185 & 60.2\% \\
    Repeated bigram & 548 & 27.9\% \\
    Repeated trigram & 273 & 13.9\% \\
    \bottomrule
    \end{tabular}
    \caption{Token-level statistics of SA-OPD filtered tokens in the visual understanding setting.}
    \label{tab:token_analysis}
    \end{table}

    We further analyze the tokens removed by SA-OPD. Specifically, we randomly sample 500 examples from visual understanding benchmarks and examine the characteristics of the filtered positions. As shown in Table~\ref{tab:token_analysis}, the filtered tokens are not merely formatting artifacts: content words account for 69.9\% of all filtered positions. Nevertheless, many of these content tokens are highly predictable from the local textual context. High-frequency filtered tokens also include punctuation and function words with very high fixedness ratios. Additional visualizations of SA-OPD-filtered tokens are provided in Appendix~E.

    Moreover, we examine the distribution of SA-OPD-filtered tokens along the student-entropy dimension. Figure~\ref{fig:beta-and-entropy}(b) shows that the filtered tokens are not confined to high-entropy regions. Instead, a substantial fraction of them correspond to low-entropy tokens that are nevertheless high-impact and weakly input-dependent. This indicates that entropy alone cannot fully capture spurious signals and shows that SA-OPD, by explicitly measuring input-groundedness, complements entropy-based selection and captures failure cases that entropy-based criteria tend to miss.

    \section{Conclusion}

    In this work, we identify spurious signals as a key unexplored failure mode in on-policy distillation: dense teacher supervision can be high-impact yet weakly grounded in the task input, triggering the student to inherit input-agnostic language priors, formatting conventions, or stereotyped reasoning templates from the teacher. To address this issue, we propose SA-OPD, a spurious-signal-aware OPD framework that estimates the input-groundedness of token-level teacher--student divergence by comparing the original input context with a residual no-prompt context. SA-OPD filters only tokens that are both weakly input-grounded and high-impact, thereby suppressing misleading dense supervision while preserving useful teacher knowledge. Experiments across VLM and LLM settings demonstrate that SA-OPD consistently improves over Vanilla OPD and competitive selective OPD baselines. Overall, our results suggest that input-groundedness is a crucial axis for reliable dense supervision in OPD, and that filtering high-impact but weakly grounded teacher signals can further improve the effectiveness and stability of on-policy distillation.

\bigskip
\bibliography{aaai2027}

\clearpage
\onecolumn
\appendix

\setcounter{secnumdepth}{2}
\renewcommand{\thesection}{\Alph{section}}
\renewcommand{\thesubsection}{\thesection.\arabic{subsection}}

  \section{Supplementary Theory}
  \label{app:theory}

  \subsection{Prior-Induced Gradients Have Weak Input-Specific Alignment}
  \label{app:prior-alignment}

Recall the sampled
  teacher--student divergence and its score-function update:
  \begin{equation}
  A_t
  =
  \log\pi_\theta(Y_t\mid X,Y_{<t})
  -
  \log\pi_{\rm T}(Y_t\mid X,Y_{<t}),
  \end{equation}

  \begin{equation}
  g_t^{\rm OPD}
  =
  -A_t s_t,
  \label{eq:app-opd-update}
  \end{equation}
  where
  \begin{equation}
  s_t
  =
  \nabla_\theta
  \log\pi_\theta(Y_t\mid X,Y_{<t}).
  \end{equation}

  Let $C_t=(Y_{<t},Y_t)$ denote the response context excluding the task input.
  We define
  \begin{equation}
  A_t^{\rm prior}
  =
  \mathbb{E}[A_t\mid C_t],
  \qquad
  A_t^{\rm grd}
  =
  A_t-A_t^{\rm prior},
  \label{eq:app-prior-grounded-definition}
  \end{equation}
  so that
  \begin{equation}
  A_t=A_t^{\rm grd}+A_t^{\rm prior},
  \qquad
  \mathbb{E}[A_t^{\rm grd}\mid C_t]=0.
  \end{equation}
  The corresponding prior-induced update is
  \begin{equation}
  g_t^{\rm prior}
  =
  -A_t^{\rm prior}s_t.
  \label{eq:app-prior-gradient}
  \end{equation}

  To isolate input-specific task improvement, define
  \begin{equation}
  \widetilde g_t^\star
  =
  g_t^\star-\mathbb{E}[g_t^\star\mid C_t],
  \qquad
  \mathbb{E}[\widetilde g_t^\star\mid C_t]=0.
  \label{eq:app-input-specific-gradient}
  \end{equation}
  In the main text, $g_t^\star$ is used as shorthand for this input-specific
  direction.

  \begin{theorem}[Weak alignment of prior-induced gradients]
  \label{thm:prior-gradient-alignment}
  Let
  \begin{equation}
  \delta s_t
  =
  s_t-\mathbb{E}[s_t\mid C_t].
  \end{equation}
  Then
  \begin{equation}
  \frac{
  \left|
  \mathbb{E}
  \left[
  \left\langle
  g_t^{\rm prior},
  \widetilde g_t^\star
  \right\rangle
  \right]
  \right|
  }{
  \sqrt{\mathbb{E}\|g_t^{\rm prior}\|^2}
  \sqrt{\mathbb{E}\|\widetilde g_t^\star\|^2}
  }
  \leq
  \kappa_t,
  \label{eq:app-normalized-alignment-bound}
  \end{equation}
  where
  \begin{equation}
  \kappa_t^2
  =
  \frac{
  \mathbb{E}
  \left[
  (A_t^{\rm prior})^2
  \|\delta s_t\|^2
  \right]
  }{
  \mathbb{E}
  \left[
  (A_t^{\rm prior})^2
  \|s_t\|^2
  \right]
  }.
  \label{eq:app-score-sensitivity-ratio}
  \end{equation}
  Hence, if the score direction is weakly input-dependent after conditioning on
  $C_t$, i.e., $\kappa_t\ll1$, the prior-induced gradient has weak
  input-specific alignment even when its gradient energy is large.
  \end{theorem}

\begin{proof}
Recall that
\begin{equation}
g_t^{\rm prior}
=
-A_t^{\rm prior}s_t .
\end{equation}
We decompose the score function into its response-context-predictable component
and input-dependent residual:
\begin{equation}
s_t
=
\bar s_t+\delta s_t,
\qquad
\bar s_t
=
\mathbb{E}[s_t\mid C_t],
\qquad
\mathbb{E}[\delta s_t\mid C_t]=0.
\label{eq:app-score-decomposition}
\end{equation}
By definition, the input-specific ideal update satisfies
\begin{equation}
\mathbb{E}[\widetilde g_t^\star \mid C_t]=0 .
\label{eq:app-ideal-centered}
\end{equation}

Since both $A_t^{\rm prior}$ and $\bar s_t$ are measurable with respect to
$C_t$, the tower property gives
\begin{align}
&
\mathbb{E}
\left[
\left\langle
-A_t^{\rm prior}\bar s_t,
\widetilde g_t^\star
\right\rangle
\right]
\nonumber\\
&=
-\mathbb{E}
\left[
A_t^{\rm prior}
\left\langle
\bar s_t,
\mathbb{E}
\left[
\widetilde g_t^\star \mid C_t
\right]
\right\rangle
\right]
=0 .
\label{eq:app-context-score-orthogonality}
\end{align}
Therefore, after substituting
$s_t=\bar s_t+\delta s_t$ into $g_t^{\rm prior}$, the alignment between the
prior-induced gradient and the input-specific task direction is
\begin{align}
\mathbb{E}
\left[
\left\langle
g_t^{\rm prior},
\widetilde g_t^\star
\right\rangle
\right]
&=
\mathbb{E}
\left[
\left\langle
-A_t^{\rm prior}(\bar s_t+\delta s_t),
\widetilde g_t^\star
\right\rangle
\right]
\nonumber\\
&=
-\mathbb{E}
\left[
A_t^{\rm prior}
\left\langle
\delta s_t,
\widetilde g_t^\star
\right\rangle
\right].
\label{eq:app-prior-residual-alignment}
\end{align}
Thus, the prior-induced gradient can align with the input-specific task
direction only through the input-dependent score residual $\delta s_t$.

Applying Cauchy--Schwarz yields
\begin{align}
&
\left|
\mathbb{E}
\left[
\left\langle
g_t^{\rm prior},
\widetilde g_t^\star
\right\rangle
\right]
\right|
\nonumber\\
&\leq
\sqrt{
\mathbb{E}
\left[
(A_t^{\rm prior})^2
\|\delta s_t\|^2
\right]
}
\sqrt{
\mathbb{E}
\left[
\|\widetilde g_t^\star\|^2
\right]
}.
\label{eq:app-prior-cauchy-schwarz}
\end{align}
Moreover, by Eq.~\eqref{eq:app-prior-gradient},
\begin{equation}
\mathbb{E}\|g_t^{\rm prior}\|^2
=
\mathbb{E}
\left[
(A_t^{\rm prior})^2
\|s_t\|^2
\right].
\label{eq:app-prior-gradient-energy}
\end{equation}
Assuming the denominators are nonzero, normalizing
Eq.~\eqref{eq:app-prior-cauchy-schwarz} by
$\sqrt{\mathbb{E}\|g_t^{\rm prior}\|^2}
\sqrt{\mathbb{E}\|\widetilde g_t^\star\|^2}$
gives
\begin{equation}
\frac{
\left|
\mathbb{E}
\left[
\left\langle
g_t^{\rm prior},
\widetilde g_t^\star
\right\rangle
\right]
\right|
}{
\sqrt{\mathbb{E}\|g_t^{\rm prior}\|^2}
\sqrt{\mathbb{E}\|\widetilde g_t^\star\|^2}
}
\leq
\sqrt{
\frac{
\mathbb{E}
\left[
(A_t^{\rm prior})^2
\|\delta s_t\|^2
\right]
}{
\mathbb{E}
\left[
(A_t^{\rm prior})^2
\|s_t\|^2
\right]
}
}.
\end{equation}
This completes the proof.
\end{proof}

  For generic or template-like tokens, the student score direction varies only
  weakly with the task input once the response context is fixed, corresponding
  to $\kappa_t\ll1$. Theorem~\ref{thm:prior-gradient-alignment} therefore gives
  \begin{equation}
  \mathbb{E}
  \left[
  \left\langle
  g_t^{\rm prior},
  \widetilde g_t^\star
  \right\rangle
  \right]
  \approx 0,
  \label{eq:app-prior-gradient-no-alignment}
  \end{equation}
  up to the scale of the two gradients. Importantly, this does not require
  \begin{equation}
  \mathbb{E}\|g_t^{\rm prior}\|^2
  =
  \mathbb{E}
  \left[
  (A_t^{\rm prior})^2\|s_t\|^2
  \right]
  \label{eq:app-prior-large-gradient-energy}
  \end{equation}
  to be small. Thus, prior-driven updates can consume substantial gradient energy
  while contributing little coherent input-specific task improvement.

  \subsection{Low-SNR OPD Updates Induce Parameter Drift}
    \label{app:parameter-drift}

    We provide an illustrative analysis of why sustained low-SNR OPD updates can
    be harmful even when their prior-driven component has approximately zero mean.

    \paragraph{Setup.}
    We use $k\in\{0,\ldots,K-1\}$ to index optimization steps. At step $k$, let
    $\mathcal{B}_k$ denote the set of valid response-token indices in the minibatch.
    For each token $t\in\mathcal{B}_k$, define the score
    \begin{equation}
    s_{k,t}
    =
    \nabla_\theta
    \log \pi_\theta(y_{k,t}\mid x_k,y_{k,<t}),
    \label{eq:appendix-token-score}
    \end{equation}
    and the sampled teacher--student divergence
    \begin{equation}
    A_{k,t}
    =
    \log\pi_\theta(y_{k,t}\mid x_k,y_{k,<t})
    -
    \log\pi_{\rm T}(y_{k,t}\mid x_k,y_{k,<t}).
    \label{eq:appendix-token-divergence}
    \end{equation}
    As in Eq.~(4), we conceptually
    decompose
    \begin{equation}
    A_{k,t}
    =
    A_{k,t}^{\rm grd}
    +
    A_{k,t}^{\rm prior},
    \label{eq:appendix-token-divergence-decomposition}
    \end{equation}
    where $A_{k,t}^{\rm grd}$ is the input-grounded component and
    $A_{k,t}^{\rm prior}$ is the prefix-predictable prior component. The
    corresponding token-level OPD update is
    \begin{equation}
    g_{k,t}^{\rm OPD}
    =
    -A_{k,t}s_{k,t}
    =
    g_{k,t}^{\rm grd}
    +
    g_{k,t}^{\rm prior},
    \label{eq:appendix-token-update-decomposition}
    \end{equation}
    with
    \begin{equation}
    g_{k,t}^{\rm grd}
    =
    -A_{k,t}^{\rm grd}s_{k,t},
    \qquad
    g_{k,t}^{\rm prior}
    =
    -A_{k,t}^{\rm prior}s_{k,t}.
    \label{eq:appendix-grounded-prior-token-updates}
    \end{equation}

    The minibatch OPD update is the average of these token-level updates:
    \begin{equation}
    G_k^{\rm OPD}
    =
    \frac{1}{|\mathcal{B}_k|}
    \sum_{t\in\mathcal{B}_k}
    g_{k,t}^{\rm OPD}
    =
    G_k^{\rm grd}
    +
    G_k^{\rm prior},
    \label{eq:batch-gradient-decomposition}
    \end{equation}
    where
    \begin{equation}
    G_k^{\rm grd}
    =
    \frac{1}{|\mathcal{B}_k|}
    \sum_{t\in\mathcal{B}_k}
    g_{k,t}^{\rm grd},
    \qquad
    G_k^{\rm prior}
    =
    \frac{1}{|\mathcal{B}_k|}
    \sum_{t\in\mathcal{B}_k}
    g_{k,t}^{\rm prior}.
    \label{eq:batch-grounded-prior-gradients}
    \end{equation}

    Let $\mathscr{H}_k$ denote the optimization history before the minibatch at
    step $k$ is sampled. We decompose the batch update into a coherent grounded
    signal and a nuisance component:
    \begin{equation}
    \mu_k
    =
    \mathbb{E}
    \left[
    G_k^{\rm grd}\mid\mathscr{H}_k
    \right],
    \qquad
    \xi_k
    =
    G_k^{\rm OPD}-\mu_k.
    \label{eq:signal-noise-batch-decomposition}
    \end{equation}
    Here, $\mu_k$ is the predictable input-grounded update, while $\xi_k$ contains
    the prior-driven update $G_k^{\rm prior}$ and the residual sampling variation
    of the grounded component.

 we then define the effective batch-level SNR as
    \begin{equation}
    \operatorname{SNR}_k
    =
    \frac{
    \|\mu_k\|^2
    }{
    \mathbb{E}
    \left[
    \|\xi_k\|^2
    \mid\mathscr{H}_k
    \right]
    }.
    \label{eq:batch-effective-snr}
    \end{equation}
    When prior-driven variation dominates the stochastic component, the denominator
    is governed mainly by the second-moment energy of $G_k^{\rm prior}$.

    \paragraph{Assumption.}
   We model the nuisance component
    as a martingale-difference sequence:
    \begin{equation}
    \mathbb{E}
    \left[
    \xi_k\mid\mathscr{H}_k
    \right]
    =0,
    \qquad
    \mathbb{E}
    \left[
    \|\xi_k\|^2\mid\mathscr{H}_k
    \right]
    =
    v_k
    <\infty.
    \label{eq:nuisance-martingale-assumption}
    \end{equation}
    This assumption does not require nuisance updates to be small. It only requires
    that, conditioned on the current optimization history, they do not provide a
    persistent directional signal. In SA-OPD, this corresponds to the regime where
    low-input-grounded, prior-driven OPD updates have little expected
    task-improving alignment but non-negligible second-moment energy.

    \begin{theorem}[Parameter drift under low-SNR OPD updates]
    \label{thm:opd-parameter-drift}
    Consider the local SGD-style recursion
    \begin{equation}
    \theta_{k+1}
    =
    \theta_k
    +
    \eta\left(\mu_k+\xi_k\right),
    \label{eq:noisy-opd-recursion}
    \end{equation}
    with learning rate $\eta>0$. Define the corresponding signal-only reference
    trajectory by
    \begin{equation}
    \bar{\theta}_{k+1}
    =
    \bar{\theta}_k
    +
    \eta\mu_k,
    \qquad
    \bar{\theta}_0=\theta_0.
    \label{eq:grounded-reference-trajectory}
    \end{equation}
    Under Assumption~\eqref{eq:nuisance-martingale-assumption}, the expected
    squared deviation from the grounded trajectory after $K$ steps is
    \begin{equation}
    \mathbb{E}
    \left[
    \left\|
    \theta_K-\bar{\theta}_K
    \right\|^2
    \right]
    =
    \eta^2
    \sum_{k=0}^{K-1}
    \mathbb{E}[v_k].
    \label{eq:parameter-drift-general}
    \end{equation}
    In particular, if $v_k=v$ for all $k$, then
    \begin{equation}
    \mathbb{E}
    \left[
    \left\|
    \theta_K-\bar{\theta}_K
    \right\|^2
    \right]
    =
    \eta^2Kv.
    \label{eq:parameter-drift-linear}
    \end{equation}
    Therefore, even zero-mean prior-driven OPD updates can induce parameter drift
    whose expected squared magnitude grows linearly with the number of optimization
    steps.
    \end{theorem}

    \begin{proof}
    Subtracting Equation~\eqref{eq:grounded-reference-trajectory} from
    Equation~\eqref{eq:noisy-opd-recursion} gives
    \begin{equation}
    \theta_{k+1}-\bar{\theta}_{k+1}
    =
    \theta_k-\bar{\theta}_k
    +
    \eta\xi_k.
    \end{equation}
    Since $\theta_0=\bar{\theta}_0$, unrolling the recursion yields
    \begin{equation}
    \theta_K-\bar{\theta}_K
    =
    \eta
    \sum_{k=0}^{K-1}
    \xi_k.
    \label{eq:unrolled-parameter-drift}
    \end{equation}
    Hence,
    \begin{align}
    \left\|
    \theta_K-\bar{\theta}_K
    \right\|^2
    &=
    \eta^2
    \left\|
    \sum_{k=0}^{K-1}
    \xi_k
    \right\|^2
    \nonumber\\
    &=
    \eta^2
    \sum_{k=0}^{K-1}
    \|\xi_k\|^2
    +
    2\eta^2
    \sum_{0\leq i<j\leq K-1}
    \left\langle
    \xi_i,\xi_j
    \right\rangle .
    \label{eq:drift-norm-expansion}
    \end{align}

    For any $i<j$, $\xi_i$ is measurable with respect to $\mathscr{H}_j$. By the
    tower property and Assumption~\eqref{eq:nuisance-martingale-assumption},
    \begin{align}
    \mathbb{E}
    \left[
    \left\langle
    \xi_i,\xi_j
    \right\rangle
    \right]
    &=
    \mathbb{E}
    \left[
    \mathbb{E}
    \left[
    \left\langle
    \xi_i,\xi_j
    \right\rangle
    \mid\mathscr{H}_j
    \right]
    \right]
    \nonumber\\
    &=
    \mathbb{E}
    \left[
    \left\langle
    \xi_i,
    \mathbb{E}
    \left[
    \xi_j\mid\mathscr{H}_j
    \right]
    \right\rangle
    \right]
    =
    0.
    \label{eq:cross-term-zero}
    \end{align}
    Taking expectation in Equation~\eqref{eq:drift-norm-expansion}, all cross terms
    vanish:
    \begin{align}
    \mathbb{E}
    \left[
    \left\|
    \theta_K-\bar{\theta}_K
    \right\|^2
    \right]
    &=
    \eta^2
    \sum_{k=0}^{K-1}
    \mathbb{E}
    \left[
    \|\xi_k\|^2
    \right]
    \nonumber\\
    &=
    \eta^2
    \sum_{k=0}^{K-1}
    \mathbb{E}
    \left[
    \mathbb{E}
    \left[
    \|\xi_k\|^2
    \mid\mathscr{H}_k
    \right]
    \right]
    \nonumber\\
    &=
    \eta^2
    \sum_{k=0}^{K-1}
    \mathbb{E}[v_k].
    \end{align}
    When $v_k=v$ for every $k$, this reduces to $\eta^2Kv$.

    \end{proof}

  \section{Experimental Details}
  \label{app:experimental-details}

  \subsection{Training Details}
  \label{app:training-details}
\paragraph{Overall Setup.}  Our implementation is built upon the verl framework. All experiments are conducted with PyTorch 2.10, CUDA 12.9, and Python 3.12. Training is performed on 8 NVIDIA H20 GPUs.

  \paragraph{Training configuration.}
  All the hyperparameters for the training are detailed in Table~\ref{tab:app-training-config}.
\begin{table*}[t]
      \centering
      \small
      \caption{Summary of training configurations for the three SA-OPD tasks.}
      \label{tab:app-training-config}
      \setlength{\tabcolsep}{3pt}
      \renewcommand{\arraystretch}{1.10}
      \begin{tabular}{p{0.23\linewidth}p{0.23\linewidth}p{0.23\linewidth}p{0.23\linewidth}}
      \toprule
      \textbf{Item}
      & \textbf{Visual Understanding}
      & \textbf{Visual Reasoning}
      & \textbf{Mathematical Reasoning} \\
      \midrule
      Main teacher
      & Qwen3.5-35B-A3B
      & Qwen3.5-35B-A3B
      & Qwen3-4B-Instruct \\
      Main student
      & Qwen3.5-2B
      & Qwen3.5-2B
      & Qwen3-1.7B \\
      Scalability teacher
      & Qwen3.5-9B
      & Qwen3.5-9B
      &  DeepSeek-R1-0528-Qwen3-8B \\
      Max prompt length  & 12000
      & 12000
      & 2048 \\
      Max response length
      & 4096
      & 4096
      & 8192 \\
      Batch size
      & 128
      & 128
      & 128 \\
      Learning rate
      & 1e-6
      & 1e-6
      & 1e-6 \\
      Optimizer
      & AdamW
      & AdamW
      & AdamW \\
      Training steps
      & 700
      & 270
      & 160 \\
      FLMR bound $\beta$
      & 1.8
      & -
      & - \\
      Initial $p_1,p_2$
      & 0.2, 0.3
      &  0.2, 0.3
      &  0.2, 0.3  \\
      \bottomrule
      \end{tabular}
    \end{table*}

   \subsection{Evaluation Details}
  \label{app:evaluation-details}

  For VLM evaluation, we use six benchmarks covering visual understanding and
  visual reasoning. Visual understanding includes EvoChart, MMIFEval, and CountQA.
  Visual reasoning includes MathVision, Geo3K, and MathVista$_{\text{mini}}$. For simplicity, we refer to MathVista${\text{mini}}$ as MathVista throughout the paper. For LLM math reasoning, we evaluate on Math500, AMC 2023, AIME 2024, AIME 2025,
  and MinervaMATH. We follow the standard exact-match or official answer
  extraction protocol for each benchmark. All benchmarks are
  evaluated using their official metrics and evaluation scripts whenever
  available.

All evaluations were conducted in a zero-shot setting. For LLM-based math
  reasoning tasks, we set the maximum number of newly generated tokens to 18,000
  and used temperature $=1.0$ with top-$p=0.95$. For VLM-based tasks, we used the
  same sampling parameters, i.e., temperature $=1.0$ and top-$p=0.95$, but adopted
  different generation lengths according to the task type: the maximum number of
  new tokens was set to 4,096 for visual reasoning tasks and 1,024 for visual
  understanding tasks.
  \begin{table}[t]
  \centering
  \caption{Evaluation benchmark summary.}
  \label{tab:app-eval-summary}
  \setlength{\tabcolsep}{6pt}
  \renewcommand{\arraystretch}{1.08}
  \begin{tabular}{lll}
  \toprule
  \textbf{Domain} & \textbf{Benchmark} & \textbf{Metric} \\
  \midrule
  Visual understanding & EvoChart & Accuracy score \\
  Visual understanding & MMIFEval & Accuracy \\
  Visual understanding & CountQA & Exact Match \\
  Visual reasoning & MathVision & Accuracy \\
  Visual reasoning & Geo3K & Accuracy \\
  Visual reasoning & MathVista & Accuracy \\
  Math reasoning & Math500 & Accuracy \\
  Math reasoning & AMC 2023 &  Avg@8 \\
  Math reasoning & AIME 2024 & Avg@8 \\
  Math reasoning & AIME 2025 & Avg@8 \\
  Math reasoning & MinervaMATH & Avg@8 \\
  \bottomrule
  \end{tabular}
  \end{table}

  \section{Detailed Algorithm}
  \label{app:detailed-algorithm}

  \paragraph{SA-OPD.} Algorithm~\ref{alg:sa-opd} summarizes the SA-OPD training procedure. The method
  differs from Vanilla OPD only in the computation of the input-grounding gap and
  the resulting token-level filtering mask.

  \begin{algorithm}[t]
  \caption{SA-OPD Training Step}
  \label{alg:sa-opd}
  \begin{algorithmic}[1]
  \Require Student policy $\pi_\theta$, teacher policy $\pi_{\rm T}$, prompt batch
  $\mathcal{B}$, initial ratios $p_1,p_2$, FLMR bound $\beta$
  \Ensure Updated student policy $\pi_\theta$
  \For{each prompt $x\in\mathcal{B}$}
      \State Sample student rollout $y=(y_1,\ldots,y_L)\sim\pi_\theta(\cdot\mid x)$
      \For{$t=1,\ldots,L$}
          \State Compute full-prompt coefficient
          \[
          A_t^{\rm full}
          =
          \log\pi_\theta(y_t\mid x,y_{<t})
          -
          \log\pi_{\rm T}(y_t\mid x,y_{<t})
          \]
          \State Compute residual no-prompt coefficient
          \[
          A_t^{\rm res}
          =
          \log\pi_\theta(y_t\mid \varnothing,y_{<t})
          -
          \log\pi_{\rm T}(y_t\mid \varnothing,y_{<t})
          \]
          \State Compute input-grounding gap
          \[
          \Delta_t^{\rm IG}=|A_t^{\rm full}-A_t^{\rm res}|
          \]
      \EndFor
  \EndFor
  \State Construct candidate set
  \[
  \mathcal{F}(p_1,p_2)
  =
  \operatorname{Bottom}_{p_1}(\Delta_t^{\rm IG})
  \cap
  \operatorname{Top}_{p_2}(|A_t^{\rm full}|)
  \]
  \State Let $\mathcal{F}$ be the final filtered token set
  \State Optimize OPD loss on retained tokens $\mathcal{V}\setminus\mathcal{F}$
  \State Update $\theta$ with the resulting gradient
  \end{algorithmic}
  \end{algorithm}

  \paragraph{Dynamic FLMR Constraint.}
  \label{app:dynamic-flmr}

  A fixed pair of selection ratios $(p_1,p_2)$ can remove very different amounts
  of loss mass across tasks and training stages. This is especially problematic
  for visual understanding, where the filtered tokens can occupy a large fraction
  of the OPD objective. We therefore constrain the filtered loss-mass ratio:
  \begin{equation}
  \mathrm{FLMR}(\mathcal{F})
  =
  \frac{
  \sum_{t\in\mathcal{F}}
  |A_t^{\rm full}|
  }{
  \sum_{t\in\mathcal{V}}
  |A_t^{\rm full}|+\epsilon
  }
  \leq \beta.
  \end{equation}

  Algorithm~\ref{alg:adaptive-filter} summarizes the adaptive FLMR-constrained filtering procedure.
  \begin{algorithm}[t]
  \caption{Adaptive FLMR-Constrained Filtering}
  \label{alg:adaptive-filter}
  \begin{algorithmic}[1]
  \Require Valid tokens $\mathcal{V}$, full-prompt coefficients $A_t^{\rm full}$,
  input-grounding gaps $\Delta_t^{\rm IG}$, initial ratios $p_1^0,p_2^0$,
  target interval $[\beta_{\min},\beta_{\max}]$
  \Ensure Filtered token set $\mathcal{F}$

  \State Define
  \[
  M=\sum_{t\in\mathcal{V}}|A_t^{\rm full}|+\epsilon .
  \]

  \Function{BuildFilter}{$\alpha$}
      \State $p_1\leftarrow \mathrm{clip}(\alpha p_1^0,0,1)$,
      $p_2\leftarrow \mathrm{clip}(\alpha p_2^0,0,1)$
      \State $\mathcal{L}\leftarrow
      \operatorname{Bottom}_{p_1}(\Delta_t^{\rm IG})$
      \State $\mathcal{E}\leftarrow
      \operatorname{Top}_{p_2}(A_t^{\rm full})
      \cup
      \operatorname{Bottom}_{p_2}(A_t^{\rm full})$
      \State $\mathcal{F}_{\alpha}\leftarrow \mathcal{L}\cap\mathcal{E}$
      \State $r(\alpha)\leftarrow
      \sum_{t\in\mathcal{F}_{\alpha}}|A_t^{\rm full}|/M$
      \State \Return $(\mathcal{F}_{\alpha},r(\alpha))$
  \EndFunction

  \State $(\mathcal{F},r)\leftarrow \Call{BuildFilter}{1}$

  \If{$r>\beta_{\max}$}
      \State Use binary search over $\alpha\in[0,1]$ to find the largest
      $\alpha$ such that $r(\alpha)\leq\beta_{\max}$.
      \State $\mathcal{F}\leftarrow\mathcal{F}_{\alpha}$
  \ElsIf{$r<\beta_{\min}$}
      \State Let $\alpha_{\max}=\max(1,1/p_1^0,1/p_2^0)$.
      \State Use binary search over $\alpha\in[1,\alpha_{\max}]$ to find the
      smallest $\alpha$ such that $r(\alpha)\geq\beta_{\min}$.
      \If{the resulting $r(\alpha)>\beta_{\max}$}
          \State Use binary search over $\alpha\in[0,\alpha_{\max}]$ to find the
          largest $\alpha$ such that $r(\alpha)\leq\beta_{\max}$.
      \EndIf
      \State $\mathcal{F}\leftarrow\mathcal{F}_{\alpha}$
  \EndIf

  \State \Return $\mathcal{F}$
  \end{algorithmic}
  \end{algorithm}
   \section{Computational Overhead}
   \label{app:computational-overhead}

  We analyze the additional computational cost introduced by SA-OPD compared with Vanilla OPD.
  SA-OPD does not introduce any new trainable parameters or auxiliary networks. Its additional cost comes from
  estimating the input-groundedness of token-level teacher--student divergence, which requires computing the same
  student-generated response under a residual no-prompt context.

  \paragraph{Cost Decomposition.}
  For a batch of on-policy responses with $N$ valid response tokens, Vanilla OPD consists of three main components:
  (i) student on-policy rollout under the original prompt,
  (ii) teacher and student log-probability evaluation under the original prompt, and
  (iii) the backward update on the OPD loss.
  Let these costs be denoted by
  \begin{equation}
  C_{\rm OPD}
  =
  C_{\rm roll}^{\rm S}(x)
  +
  C_{\rm eval}^{\rm S}(x,y)
  +
  C_{\rm eval}^{\rm T}(x,y)
  +
  C_{\rm bwd},
  \label{eq:opd-cost}
  \end{equation}
  where $C_{\rm roll}^{\rm S}(x)$ is the autoregressive student rollout cost,
  $C_{\rm eval}^{\rm S}(x,y)$ and $C_{\rm eval}^{\rm T}(x,y)$ are the student and teacher log-probability evaluation
  costs on the sampled response, and $C_{\rm bwd}$ is the backward and optimizer-update cost. SA-OPD reuses the same on-policy response $y$ generated by the student. To estimate the input-grounding gap $
  \Delta_t^{\rm IG}$, it additionally computes the teacher--student divergence under a residual no-prompt context:
  \begin{equation}
  A_t^{\rm res}
  =
  \log\pi_\theta(y_t\mid \varnothing,y_{<t})
  -
  \log\pi_{\rm T}(y_t\mid \varnothing,y_{<t}).
  \end{equation}
  This requires one additional student pass and one additional teacher pass under the residual context. Therefore, the
  total cost of SA-OPD can be written as
  \begin{equation}
  C_{\rm SA}
  =
  C_{\rm OPD}
  +
  C_{\rm eval}^{\rm S}(\varnothing,y)
  +
  C_{\rm eval}^{\rm T}(\varnothing,y)
,
  \label{eq:sa-cost}
  \end{equation}
 \paragraph{Empirical Analysis.}
  We further quantify the practical training overhead of SA-OPD against Vanilla
  OPD under the same experimental settings used in the main results. As shown in
  Table~\ref{tab:training_efficiency}, SA-OPD introduces only modest additional
  cost across tasks, with the measured overhead ranges from 2.64\% to 7.53\%.

  This overhead mainly comes from the residual no-prompt divergence computation
  used for spurious-signal detection. Specifically, SA-OPD performs additional
  teacher and student scoring under the residual no-prompt context to estimate
  token-level input-groundedness. These computations are used only to construct
  the filtering mask and do not introduce extra trainable parameters or additional
  backward passes. Overall, SA-OPD provides a controlled efficiency--reliability
  trade-off: a small amount of extra scoring computation in exchange for more reliable dense OPD supervision.

   \begin{table}[t]
      \centering
      \setlength{\tabcolsep}{6pt}
      \renewcommand{\arraystretch}{1.15}
      \begin{tabular}{llcc}
      \toprule
      \textbf{Task / Hardware}
      & \textbf{Method}
      & \makecell{\textbf{Total Training Time}\\\textbf{(hours)}}
      & \textbf{Overhead} \\
      \midrule
      \multirow{2}{*}{
      \makecell{\textbf{Math Reasoning}\\
      \textbf{(8$\times$H20)}
      }}
      & OPD
      & 5.19
      & -- \\
      & \textbf{SA-OPD}
      & \textbf{5.54}
      & \textbf{+6.74\%} \\
      \midrule
      \multirow{2}{*}{
      \makecell{\textbf{Visual Understanding}\\
      \textbf{(8$\times$H20)}
      }}
      & OPD
      & 4.38
      & -- \\
      & \textbf{SA-OPD}
      & \textbf{4.71}
      & \textbf{+7.53\%} \\
      \midrule
      \multirow{2}{*}{
      \makecell{\textbf{Visual Reasoning}\\
      \textbf{(8$\times$H20)}
      }}
      & OPD
      & 1.89
      & -- \\
      & \textbf{SA-OPD}
      & \textbf{1.94}
      & \textbf{+2.64\%} \\
      \bottomrule
      \end{tabular}
      \caption{
      Training efficiency comparison between Vanilla OPD and SA-OPD across
      different tasks. SA-OPD incurs only modest overhead from residual no-prompt
      divergence computation and does not introduce additional trainable
      parameters or extra backward passes.
      }
      \label{tab:training_efficiency}
  \end{table}

 \newpage
  \section{Filtered Token Example}
  \label{app:case-studies}

  We present qualitative visualizations of SA-OPD filtered tokens in this section. For each example, the generated response is decomposed into tokens, and each token is annotated with two values. \textbf{F} represents $A_t^{\rm full}$, the token-level teacher--student divergence under the original input context, while $\Delta$ represents the Input-Grounding Gap, which quantifies the sensitivity of this divergence to the task input. Tokens outlined in red are selected by SA-OPD for filtering. These tokens have large optimization impact but small input-grounding gap, indicating that their teacher supervision is more likely driven by language priors or template preferences rather than input-specific evidence.

 \begin{figure*}[t]
        \centering
        \includegraphics[width=1.0\textwidth]{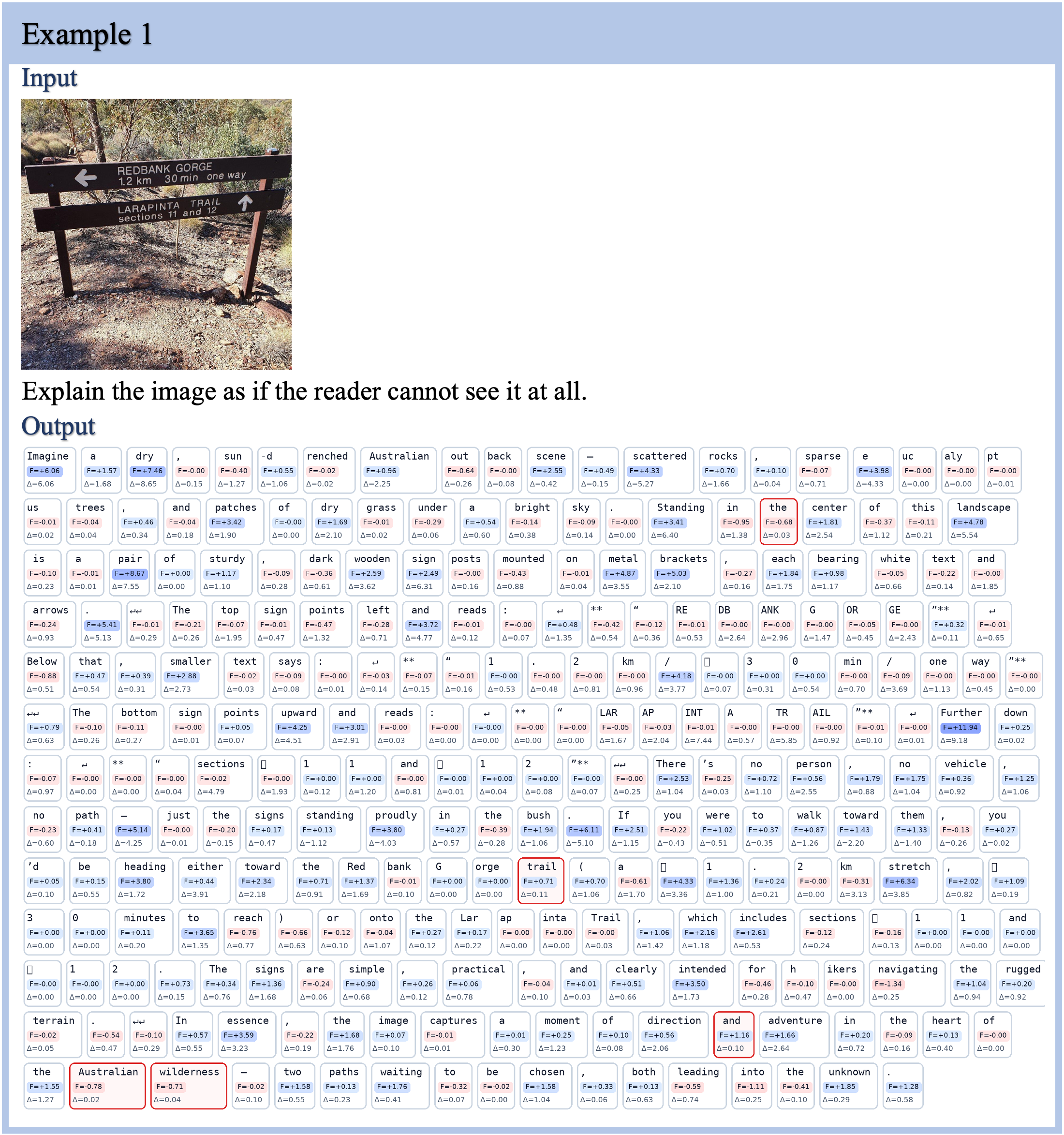}

        \label{fig:e1}
    \end{figure*}
     \begin{figure*}[t]
        \centering
        \includegraphics[width=1.0\textwidth]{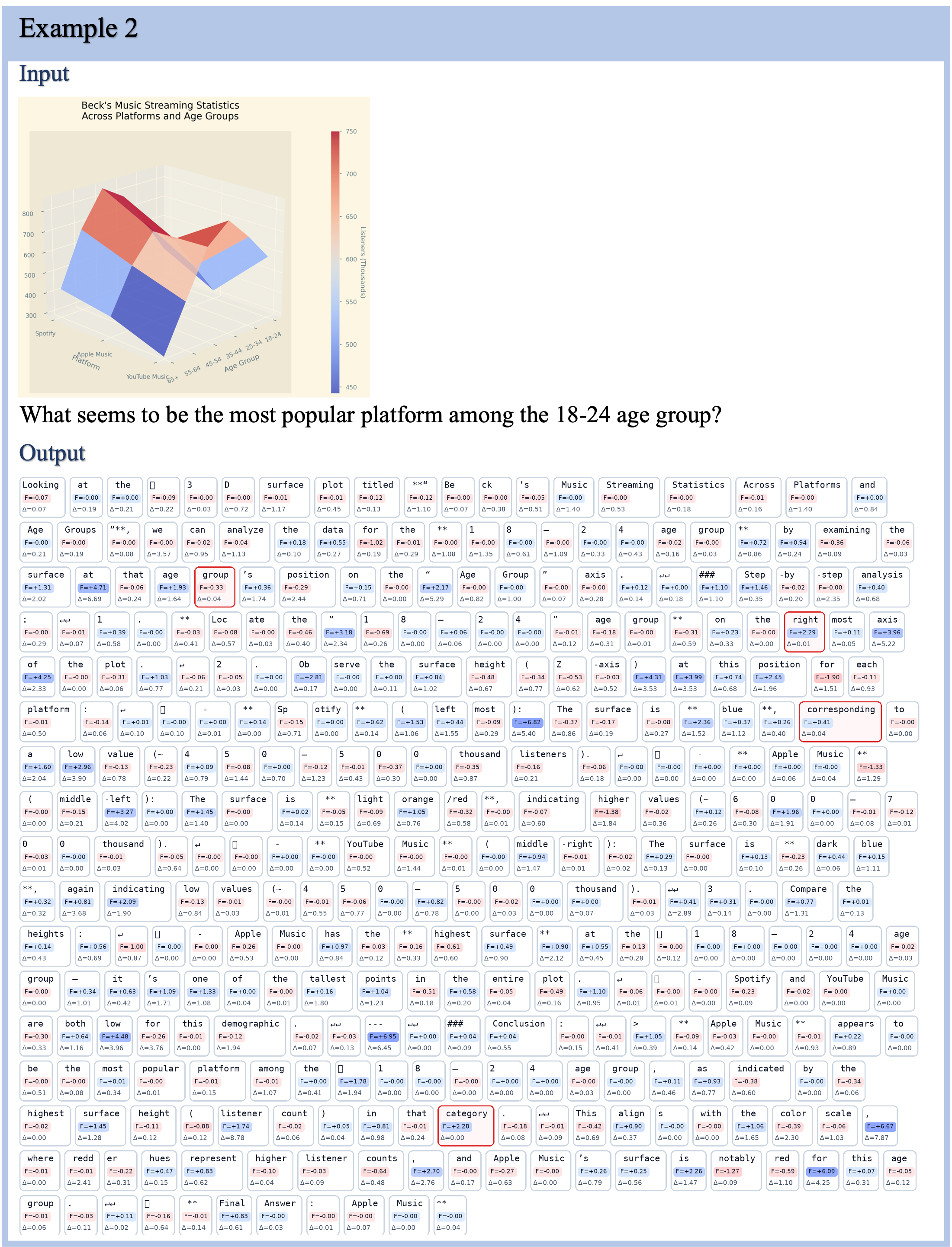}

        \label{fig:e2}
    \end{figure*}
      \begin{figure*}[t]
        \centering
        \includegraphics[width=1.0\textwidth]{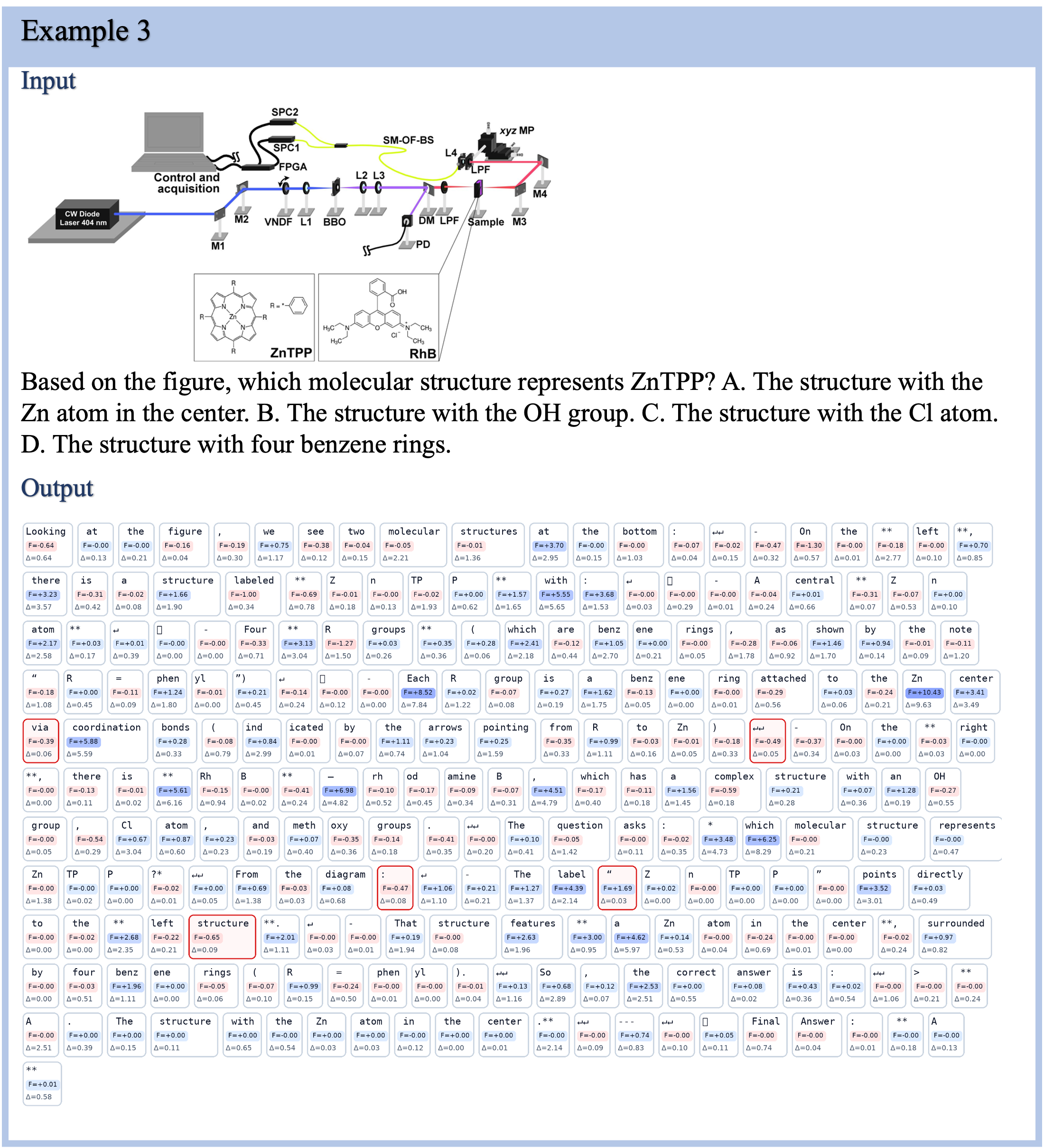}

        \label{fig:e3}
    \end{figure*}
      \begin{figure*}[t]
        \centering
        \includegraphics[width=1.0\textwidth]{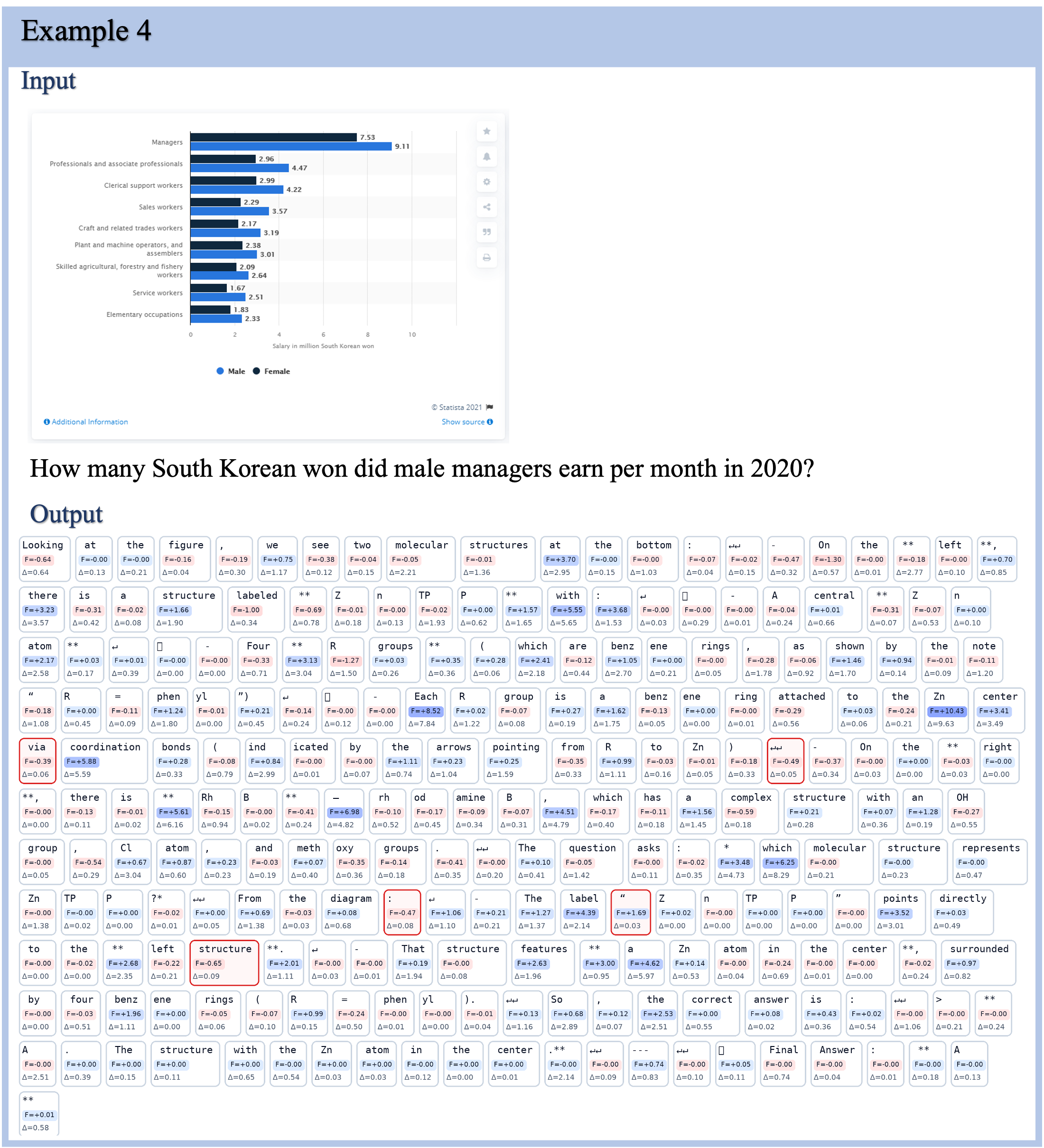}

        \label{fig:e4}
    \end{figure*}
      \begin{figure*}[t]
        \centering
        \includegraphics[width=0.85\textwidth]{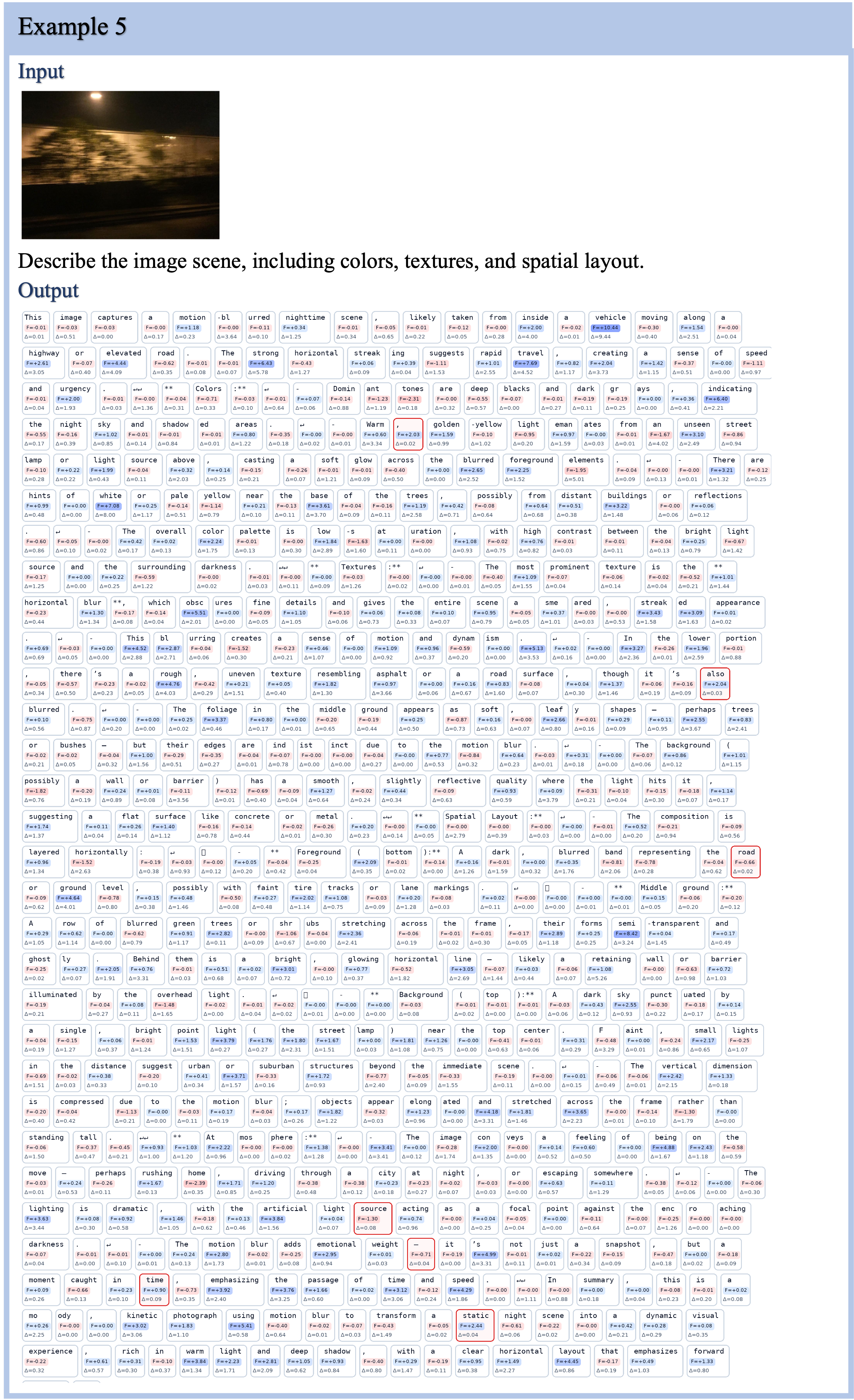}

        \label{fig:e5}
    \end{figure*}
\end{document}